\PassOptionsToPackage{table}{xcolor}
\documentclass{article}

\usepackage{graphicx}
\usepackage{booktabs}
\usepackage{makecell}
\usepackage{xcolor}
\usepackage{colortbl}
\usepackage{threeparttable}

 \usepackage[preprint]{neurips_2026}

\usepackage[utf8]{inputenc} 
\usepackage[T1]{fontenc}    
\usepackage{hyperref}       
\usepackage{url}            
\usepackage{booktabs}       
\usepackage{amsfonts}       
\usepackage{nicefrac}       
\usepackage{microtype}      
\usepackage{xcolor}         

\usepackage{amsmath}
\usepackage{amssymb}
\usepackage{mathtools}
\usepackage{amsthm}

\usepackage{graphicx}      
\usepackage{booktabs}
\usepackage{multirow}
\usepackage{makecell}
\usepackage[table]{xcolor}
\usepackage{threeparttable}

\definecolor{bestgreen}{HTML}{D9EAD3}
\definecolor{secondblue}{HTML}{D9EAF7}
\definecolor{methodyellow}{HTML}{FFF2CC}
\definecolor{lightgray}{HTML}{F2F2F2}

\newcommand{\best}[1]{\cellcolor{bestgreen}\textbf{#1}}
\newcommand{\second}[1]{\cellcolor{secondblue}\underline{#1}}

\newcommand{\crit}[1]{\cellcolor{bestgreen}\textbf{#1}}

\newcommand{\oursrow}{\rowcolor{methodyellow}}

\usepackage[capitalize,noabbrev]{cleveref}

\theoremstyle{plain}
\newtheorem{theorem}{Theorem}[section]
\newtheorem{proposition}[theorem]{Proposition}
\newtheorem{lemma}[theorem]{Lemma}

\theoremstyle{definition}
\newtheorem{definition}[theorem]{Definition}
\newtheorem{assumption}[theorem]{Assumption}
\theoremstyle{remark}
\newtheorem{remark}[theorem]{Remark}

\usepackage{algorithm}
\usepackage{algpseudocode}
\usepackage{amsmath}
\usepackage{amsthm}

\title{Graph-Regularized Synthetic Control for Stable Counterfactual Estimation under Donor Redundancy}

\author{%
  Yanlin Zhang \\
  University of Illinois Urbana-Champaign\\
  \textit{yanlinz4@illinois.edu}
  \And
  Sungyong Chung \\
  University of Illinois Urbana-Champaign \\
  \textit{sc151@illinois.edu}
  \And
  Alireza Talebpour \\
    University of Illinois Urbana-Champaign \\
    \textit{ataleb@illinois.edu}\\
}

\begin{document}

\maketitle

\begin{abstract}
Synthetic control methods estimate counterfactual outcomes by weighting untreated donor units to match a treated unit before treatment. In recent applications, however, the donor pool can be large, highly correlated, and short in pre-treatment history, yielding many near-equivalent solutions with unstable and difficult-to-interpret donor weights. We propose Graph-Regularized Synthetic Control (Graph-SC), a convex estimator that augments the classical pre-treatment matching objective with a graph Laplacian penalty. The graph encodes similarity among donor trajectories, encouraging coherent weights over structurally related donors while preserving the transparency of synthetic controls. We provide theoretical results characterizing the bias-variance tradeoff induced by graph regularization and establish efficient convergence guarantees with an accelerated proximal gradient method. Across a controlled simulation under donor redundancy, the California Proposition 99 benchmark, and observational counterfactual analyses on the Waymo Open Motion Dataset, Graph-SC improves weight stability, reduces sensitivity to redundant donor pools, and maintains competitive pre- and post-treatment accuracy relative to classical and regularized synthetic control baselines. Our results suggest that graph regularization is most valuable not as a universal accuracy booster, but as a principled mechanism for selecting stable and interpretable counterfactuals from near-equivalent synthetic controls with comparable pre-treatment fit.
\end{abstract}

\section{Introduction}
Synthetic control (SC) is arguably the most important
Innovation in the policy evaluation literature \cite{athey2017state}. It estimates the counterfactual trajectory of a treated unit (e.g., a region or individual exposed to an intervention) by forming a weighted combination of untreated donors (e.g., comparable regions or individuals not exposed to the intervention) whose pre-intervention behavior matches the treated unit \cite{abadie2003economic,abadie2010synthetic}. Its appeal lies in its interpretability, since the counterfactual is constructed as an explicit mixture of real units, and in its simplicity, as treatment effects are computed directly from post-intervention discrepancies.

A fundamental challenge arises in modern, data-rich applications with large and highly correlated donor pools. When many donors exhibit similar pre-intervention behavior, the weight learning problem admits multiple near equivalent solutions that achieve indistinguishable pre-fit. As a result, the estimation becomes ill-conditioned, with learned weights that are overly sensitive to small perturbations such as noise, resampling, window selection, or feature scaling. This sensitivity propagates to the counterfactual, yielding high variance estimates and unstable treatment effects. The central question motivating this work is whether counterfactual construction can be stabilized without sacrificing the transparency of donor weight mixtures.

We propose \textbf{Graph Synthetic Control (Graph-SC)}, a convex estimator that incorporates pre-intervention donor structure into synthetic control through a learned similarity graph. The donor graph is constructed from pre-intervention features, and the weight vector is regularized via a Laplacian quadratic form. Intuitively, donors who behave similarly before treatment should receive similar weights. Formally, the graph regularizer couples weights along graph edges, selecting stable solutions among many near equivalent fits while preserving interpretability. Graph-SC further combines this structural prior with sparsity, enforced via an $\ell_1$ penalty, yielding stable donor mixtures and a well-conditioned optimization problem. Our Contribution can be summarized as follows.
\begin{itemize}
    \item \textbf{Estimator.} We introduce Graph-SC, a graph-regularized synthetic control method that combines pre-intervention matching, $\ell_1$ sparsity, and Laplacian smoothing on a donor similarity graph learned from pre-intervention data.
    \item \textbf{Theory.} 
    We characterize how graph regularization affects stability through a bias-variance tradeoff. We derive finite sample upper bounds on estimation error that quantify variance reduction from graph smoothing, and a complementary lower bound showing that when the learned graph is misspecified, the induced bias is unavoidable. These results delineate when Graph-SC outperforms classical synthetic control.
    \item \textbf{Computation.} We show Graph-SC is efficiently solvable by (accelerated) proximal gradient methods \citep{beck2009fast} and establish convergence rates specialized to the Graph-SC composite objective.
\end{itemize}
\section{Related Work and Setting}
\subsection{Related Work}
\label{subsec:related_work}
\paragraph{Synthetic control and panel counterfactual.}
Classical SC was introduced for comparative case studies in economics and policy evaluation \cite{abadie2003economic,abadie2010synthetic}. Subsequent work connected SC to broader panel estimators and improved robustness via reweighting, including augmented SC \cite{ben2021augmented}, synthetic difference-in-differences \cite{arkhangelsky2021synthetic}, interactive fixed-effect imputation \cite{xu2017generalized}, and low-rank/matrix-denoising approaches for noisy panels \cite{amjad2018robust}. These methods primarily regularize the structure of the outcome panel, whereas our approach is complementary, regularizing the donor weights directly. By learning a similarity graph from pre-intervention data and applying graph-based smoothing to the weights, we address instability arising from large, highly correlated donor pools.

\paragraph{Graph Regularization}
Regularized variants of SC modify the original constraints or objective to improve performance with many donors, including elastic-net viewpoints and penalized SCM \cite{abadie2021penalized}. Graph-structured regularization, including Laplacian-based smoothing and network-guided penalties, is well studied in regression and clustering \citep{li2008network, hallac2015network}. Graph-SC adapts these ideas to counterfactual construction by encoding a pre-intervention similarity prior over donors through a graph. This induces structured, group-aware donor selection, improves stability, and preserves the counterfactual as an interpretable mixture of real donors.

\subsection{Problem setup}
\label{subsec:setup}

We consider a standard comparative case study with one treated unit (index $1$) and a donor pool of $J$ untreated units (indices $2,\dots,J{+}1$). Time is indexed by $t\in\{1,\dots,T\}$, with an intervention at $T_0$ (pre-intervention $t\le T_0$, post-intervention $t>T_0$). Let $y_{it}(0)$ and $y_{it}(1)$ denote the potential outcomes without and with treatment.

\begin{definition}
\label{def:sc}
Classical synthetic control (SC) \cite{abadie2003economic} estimates donor weights as a convex combination that matches pre-intervention features:
\begin{equation}
\label{eq:sc}
\widehat w_{\mathrm{SC}}
\in
\arg\min_{w\in\mathbb{R}^J}
\frac12\|x_1-X_0 w\|_2^2
\quad
\text{s.t.}\quad
w\ge 0,\;\mathbf{1}^\top w=1.
\end{equation}
The counterfactual is $\widehat y_{1t}(0)=Y_{0t}^\top \widehat w_{\mathrm{SC}}$ for $t>T_0$.
\end{definition}

Classical SC is appealing because it is simple and transparent. However, when $J$ is large and donors are correlated, there may be many distinct weight vectors that achieve essentially the same pre-intervention fit, which makes the counterfactual sensitive to noise and modeling choices.

\begin{proposition}
\label{prop:nonunique}
Suppose there exists a nonzero vector $v\in\mathbb{R}^J$ such that $X_0 v=0$ and $\mathbf{1}^\top v=0$. If $\widehat w$ is feasible for \eqref{eq:sc}, then for all sufficiently small $t\in\mathbb{R}$ such that $\widehat w + t v\ge 0$, the perturbed weight vector $\widehat w + t v$ is also feasible and attains the same objective value:
\[
\|x_1-X_0(\widehat w+t v)\|_2^2 = \|x_1-X_0\widehat w\|_2^2.
\]
Hence, the SC solution is not unique. The proof is provided in subsection ~\ref{proof:nonunique}.
\end{proposition}

\begin{remark}
\label{rem:illposed}
In practice, donors are often nearly collinear, so making $X_0^\top X_0$ ill-conditioned. This creates nearly flat directions in the SC objective: small pre-intervention noise can induce large changes in $\widehat w_{\mathrm{SC}}$, inflating the variance of post-intervention effects. The rest of our method introduces a structured inductive bias to select among many near-equivalent pre-fit solutions.
\end{remark}

\section{Graph-Regularized Synthetic Control}
\subsection{Learning donor similarity graphs}
\label{:graph_learning}

\begin{definition}
\label{def:graph}
Let the donor pool be indexed by $u\in\{1,\dots,J\}$, corresponding to original units $j=u+1$ with pre-intervention feature vectors $x_{u+1}\in\mathbb{R}^k$. Fix a symmetric similarity function $s:\mathbb{R}^k\times\mathbb{R}^k\to\mathbb{R}$ computed using pre-intervention data only, and a neighborhood size $k_{\mathrm{nn}}\in\mathbb{N}$. Define the similarity matrix $S\in\mathbb{R}^{J\times J}$ by
\[
S_{uv}=
\begin{cases}
s(x_{u+1},x_{v+1}), & u\neq v,\\
-\infty, & u=v.
\end{cases}
\]
For each $u$, let $\mathcal{N}(u):=\mathrm{TopKIndices}(S_{u,:},k_{\mathrm{nn}})$ be the indices of the $k_{\mathrm{nn}}$ largest entries of row $u$. Define a directed $k$NN adjacency $A^{\mathrm{dir}}\in\{0,1\}^{J\times J}$ by $A^{\mathrm{dir}}_{uv}=\mathbf{1}\{v\in\mathcal{N}(u)\}$, then symmetrize to obtain an undirected unweighted adjacency
\[
A \;:=\; \mathbf{1}\{A^{\mathrm{dir}}+(A^{\mathrm{dir}})^\top>0\},
\qquad A_{uu}=0.
\]
Let $d_u=\sum_{v=1}^J A_{uv}$ and $D=\mathrm{diag}(d_1,\dots,d_J)$. The normalized Laplacian is
\[
L \;:=\; I - D^{-1/2} A D^{-1/2},
\]
with the convention $(D^{-1/2})_{uu}=0$ if $d_u=0$.
\end{definition}

\paragraph{Graph Implementation.} \label{rem:graph_complexity}Algorithm~\ref{alg:graph_learning} summarizes the implementation. It first computes pairwise similarities between donors, then forms a $k_{\mathrm{nn}}$-nearest-neighbor graph, and finally symmetrizes it to ensure the downstream regularizer is a positive semidefinite quadratic form. Graph construction is $O(J^2k)$ naively and can be made near-linear in $J$ with details in appendix~\ref{proof:graph_complexity}.

\begin{algorithm}[h]
\caption{Donor Similarity Graph Learning}
\label{alg:graph_learning}
\begin{algorithmic}[1]
\Require Donor pre-intervention features $\{x_j\}_{j=2}^{J+1}$, similarity function $s(\cdot,\cdot)$, neighborhood size $k_{\mathrm{nn}}$
\Ensure Normalized graph Laplacian $L$

\State Initialize $A^{\mathrm{dir}}\gets 0\in\mathbb{R}^{J\times J}$
\State Compute pairwise similarities $S_{uv}\gets s(x_{u+1},x_{v+1})$ for all $u\neq v$, and set $S_{uu}\gets -\infty$

\For{$u=1,\ldots,J$}
    \State $\mathcal{N}(u)\gets \mathrm{TopKIndices}(S_{u,:}, k_{\mathrm{nn}})$
    \For{each $v\in\mathcal{N}(u)$}
        \State $A^{\mathrm{dir}}_{uv}\gets 1$
    \EndFor
\EndFor

\State $A \gets \mathbf{1}\{A^{\mathrm{dir}}+(A^{\mathrm{dir}})^\top>0\}$
\State $D \gets \mathrm{diag}(d_1,\ldots,d_J)$, where $d_u=\sum_{v=1}^J A_{uv}$
\State $L \gets I - D^{-1/2} A D^{-1/2}$
\end{algorithmic}
\end{algorithm}

\begin{lemma}[Edge-wise interpretation of Laplacian smoothness]
\label{lem:laplacian_identity}
For the normalized Laplacian $L$ in Definition~\ref{def:graph},
\[
\begin{aligned}
w^\top L w
&= \frac{1}{2}\sum_{u=1}^{J}\sum_{v=1}^{J} A_{uv}
\left(
\frac{w_u}{\sqrt{d_u}}-\frac{w_v}{\sqrt{d_v}}
\right)^2 \\
&= \sum_{(u,v)\in E}
\left(
\frac{w_u}{\sqrt{d_u}}-\frac{w_v}{\sqrt{d_v}}
\right)^2 ,
\end{aligned}
\]
where \(E=\{(u,v):u<v,\ A_{uv}=1\}\).

The proof is provided in appendix~\ref{proof:laplacian_identity}
\end{lemma}

\begin{remark}
\label{rem:graph_prior}
Lemma~\ref{lem:laplacian_identity} shows that $w^\top L w$ penalizes edge-wise difference between weights of similar donors. Therefore, it selects among many pre-fit solutions by favoring donor subnetworks, meaning that if donors are connected in the similarity graph, large weight differences are discouraged.
\end{remark}


\subsection{Graph-Regularized Synthetic Control Estimator}

We incorporate donor structure into synthetic control by introducing a graph-based regularizer that acts as a soft prior: donors that exhibit similar pre-intervention behavior should receive similar weights. This yields a convex estimator that preserves the transparency of weight-based counterfactual construction, while avoiding brittle hard constraints such as exact sum-to-one.

This section addresses the following research question: With large, highly correlated donor pools, can a similarity-induced graph prior reduce the variance of the estimated counterfactual and stabilize treatment effect estimates? Moreover, when does the graph term reduce variance and improve stability, and what bias can it introduce?

\begin{definition}
\label{def:gsc}
Given $x_1$, $X_0$, and Laplacian $L$, define Graph Synthetic Control (Graph-SC) and estimate weights by
\begin{equation}
\label{eq:gsc}
\begin{aligned}
\widehat w
\in \arg\min_{w\in\mathbb{R}^J}
\;&
\frac12\|x_1-X_0 w\|_2^2
+ \lambda_1\|w\|_1 \\
&\quad
+ \frac{\lambda_g}{2} w^\top L w
+ \frac{\lambda_{sum}}{2}(\mathbf{1}^\top w-1)^2 .
\end{aligned}
\end{equation}
with tuning parameters $\lambda_1,\lambda_{g},\lambda_{sum}\ge 0$.
\end{definition}



\begin{assumption}[Similarity graph]
\label{as:graph_connected}
The donor similarity graph is undirected and connected, and $L$ is its Laplacian.
\end{assumption}

\begin{assumption}[Pre-intervention noise model]
\label{as:pre_noise}
There exists a target weight vector $w^\star\in\mathbb{R}^J$ with $w^\star\ge 0$
such that
\[
x_1 = X_0 w^\star + \varepsilon,\qquad \mathbb{E}[\varepsilon]=0,\qquad
\mathbb{E}[\varepsilon\varepsilon^\top]=\sigma_x^2 I_K,
\]
and $\varepsilon$ is sub-Gaussian.
\end{assumption}

\begin{definition}
\label{def:mu}
Define the quadratic regularizer Hessian
\[
Q \;\triangleq\; \lambda_g L + \lambda_{sum} \mathbf{1}\mathbf{1}^\top.
\]
Let $\lambda_2(L)$ denote the second-smallest eigenvalue of $L$.
Under Assumption~\ref{as:graph_connected}, $L$ has eigenvalues
$0=\lambda_1(L)<\lambda_2(L)\le \cdots$.
Define the \emph{effective curvature}
\begin{equation}
\label{eq:mu_def}
\mu \;\triangleq\; \lambda_{\min}(Q)
\;=\; \min\{\lambda_g\lambda_2(L),\;\lambda_{sum} J\}.
\end{equation}
\end{definition}

\paragraph{Bias--variance trade-off.} Variance reduction comes from graph regularization, and bias arises from shrinkage toward graph-smooth weights. We formalize this with an ``optimal'' weight vector $w^\star$ as the target that matches the noiseless pre-intervention features and a population minimizer $\bar w$ defined by replacing $x_1$ with its conditional mean $X_0 w^\star$.

\begin{definition}
\label{def:population_minimizer}
\begin{equation}
\label{eq:population_minimizer}
\begin{aligned}
\bar w
\in \arg\min
\;&
\frac12\|X_0 w^\star - X_0 w\|_2^2
+ \lambda_1 \|w\|_1 \\
&\quad
+ \frac{\lambda_g}{2} w^\top L w
+ \frac{\lambda_{sum}}{2}(\mathbf{1}^\top w - 1)^2 .
\end{aligned}
\end{equation}

The regularization bias is $\bar w - w^\star$.
\end{definition}

\begin{theorem}[Bias and variance bounds]
\label{thm:bias_variance_tradeoff}
Under Assumptions~\ref{as:graph_connected}--\ref{as:pre_noise}, let $\mu$ be as in
Definition~\ref{def:mu}. Then the estimation error admits the decomposition.
\[
\widehat w - w^\star = (\widehat w - \bar w) + (\bar w - w^\star),
\]
with the following bounds.

\smallskip
\noindent\textbf{Variance term upper bound:}
\begin{equation}
\label{eq:estimation_term_bound}
\mathbb{E}\big[\|\widehat w - \bar w\|_2^2\big]
\;\le\;
\frac{\sigma_x^2}{\mu^2}\,\mathrm{tr}(X_0^\top X_0).
\end{equation}
\smallskip
\noindent\textbf{Bias term upper bound:}
There exists a subgradient $\xi\in\partial \| \bar w\|_1$ with $\|\xi\|_\infty\le 1$
such that
\begin{equation}
\label{eq:bias_upper}
\begin{aligned}
\|\bar w - w^\star\|_2
\;\le\;
\frac{1}{\mu}\Big(
&\lambda_1\|\xi\|_2 + \lambda_g\|L w^\star\|_2 + \\&\lambda_{sum}|\mathbf{1}^\top w^\star-1|\sqrt{J}
\Big).
\end{aligned}
\end{equation}
In particular, if $\mathbf{1}^\top w^\star=1$ then the last term vanishes and the bias is controlled by the graph-misspecification measure $\|L w^\star\|_2$.
\smallskip

\noindent\textbf{Bias term lower bound under graph misspecification.}
In the graph-only case ($\lambda_1=0$) with $\mathbf{1}^\top w^\star=1$,
\begin{equation}
\label{eq:bias_lower}
\|\bar w - w^\star\|_2
\;\ge\;
\frac{\lambda_g}{\|X_0^\top X_0 + Q\|_{\mathrm{op}}}\,
\|P_{\mathbf{1}^\perp}(L w^\star)\|_2,
\end{equation}
where $P_{\mathbf{1}^\perp}$ projects onto the subspace orthogonal to $\mathbf{1}$.

The proof is provided in appendix~\ref{proof:bias_variance_tradeoff}
\end{theorem}

\begin{remark}
    Graph regularization reduces variance at rate $\mu^{-2}$ via \eqref{eq:estimation_term_bound} but introduces bias proportional to $\lambda_2\|Lw^\star\|_2/\mu$ via \eqref{eq:bias_upper}. Therefore, Graph-SC is most beneficial when $w^\star$ is smooth (i.e., $\|Lw^\star\|_2$ is small) and the donor design is ill-conditioned (i.e., large $\mathrm{tr}(X_0^\top X_0)$), which is the case where classical SC is unstable. Sparse-only ($\lambda_g=0$) can still be unstable when donors are highly correlated (multiple sparse mixtures fit similarly), while graph-only ($\lambda_1=0$) can spread weights over a neighborhood and introduce bias if the similarity graph is misspecified. Therefore, the combined objective balances these effects: $\lambda_1$ controls sparsity and $\lambda_g$ stabilizes and groups along the graph.
\end{remark}

\subsection{Fast optimization for Graph-SC via proximal gradient}
Graph-SC is a convex composite problem of the form $\min_{w\in\mathbb{R}^J}\; F(w)\triangleq f(w)+g(w)$, where $f$ is smooth with Lipschitz gradient and $g$ is proximable.
This places Graph-SC directly in the standard setting of proximal gradient methods (FISTA) \cite{beck2009fast}, and can achieve a convergence rate of  \(\mathcal{O}(1/k^2)\). The corresponding setup and proof can be found in appendix ~\ref{proof:conv_gsc_compact}.

\section{Experiments}\label{sec:experiments}
We evaluate Graph-Regularized Synthetic Control (Graph-SC) across three settings: a controlled semi-synthetic benchmark with known treatment effects, the California Proposition 99 policy benchmark, and trajectory-level observational counterfactual analysis using the Waymo Open Motion Dataset (WOMD). The experiments are designed around a central question: when many donor combinations produce similar pre-treatment fit, can graph regularization select a more stable and interpretable synthetic control? We therefore report not only conventional pre-treatment fit and post-treatment discrepancy, but also donor-weight concentration, graph coherence, placebo behavior, and sensitivity to donor-pool perturbations. 

\subsection{Semi-Synthetic Study: Donor Redundancy and Bias--Variance Trade-off}\label{subsec:simulation}

We first study a setting where the true counterfactual and treatment effect are known. The purpose is not to mimic a particular empirical dataset, but to isolate the estimator behavior under donor redundancy. We generate donor trajectories from latent smooth clusters so that many donors are individually correlated with the treated unit, and many donor combinations achieve similar pre-treatment fit. This setting creates a synthetic-control Rashomon set: multiple weight vectors have nearly identical pre-treatment loss but can imply different post-treatment counterfactuals.

\noindent\textbf{Data Generation.}
We simulate $N=200$ donor units and one treated unit over $T=70$ periods, with $T_0=50$ pre-treatment and $T_1=20$ post-treatment periods. Donors are assigned to $K=10$ latent clusters. Each cluster trajectory is generated from a low-frequency Fourier basis plus a linear trend, and unit-level trajectories are obtained by adding Gaussian noise. The treated unit is generated as a convex combination of donors from one latent cluster and receives a constant additive treatment effect $\tau=2.0$ after $T_0$.

\noindent\textbf{Estimators.} We compare standard SC, graph-only regularization, sparse-only regularization, and the full Graph-SC objective. The graph is constructed from donor pre-treatment trajectories using $k$-nearest neighbors unless otherwise specified. In this controlled setting, we additionally evaluate recovery of the true donor cluster and stability under donor resampling.

\noindent\textbf{Results.}
Figure~\ref{fig:simulation_results} shows the performance metrics across 200 repetitions, revealing distinct failure modes in the baselines that Graph-SC effectively overcomes. The Graph Regularization Only estimator (Orange) exhibits extreme instability with a variance of approximately 20.24, frequently overestimating the effect magnitude, which confirms that graph smoothness alone is insufficient for conditioning the problem without sparsity. Conversely, the Sparse Regularization Only estimator (Green) successfully reduces variance to approximately 0.44 but suffers from severe shrinking bias of -1.93, underestimating the true effect of 2.0. The Graph-SC (Red) achieves the optimal tradeoff, yielding the lowest Mean Squared Error (MSE) of 0.30 by balancing these objectives. As shown in Figure~\ref{fig:simulation_results} (Bottom), Graph-SC reduces the bias of the sparse estimator by approximately 83\% (from -1.93 to -0.33) while simultaneously maintaining a low variance (0.19) that is significantly better than the baselines.

\begin{figure}[t]
    \centering

    \begin{minipage}[t]{0.48\linewidth}
        \centering
        \includegraphics[width=\linewidth]{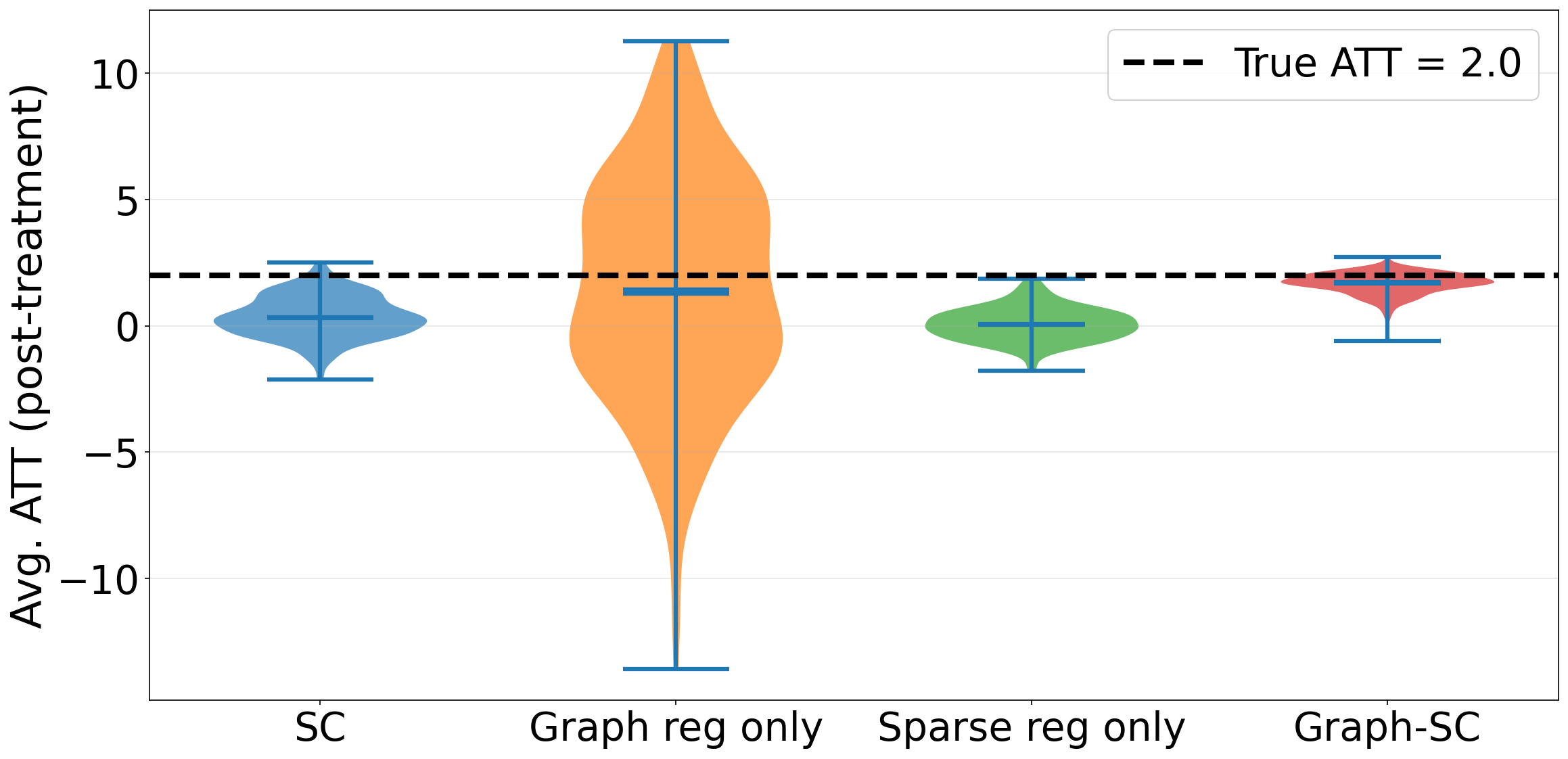}
        \textbf{(a)} Effect distribution
    \end{minipage}
    \hfill
    \begin{minipage}[t]{0.48\linewidth}
        \centering
        \includegraphics[width=\linewidth]{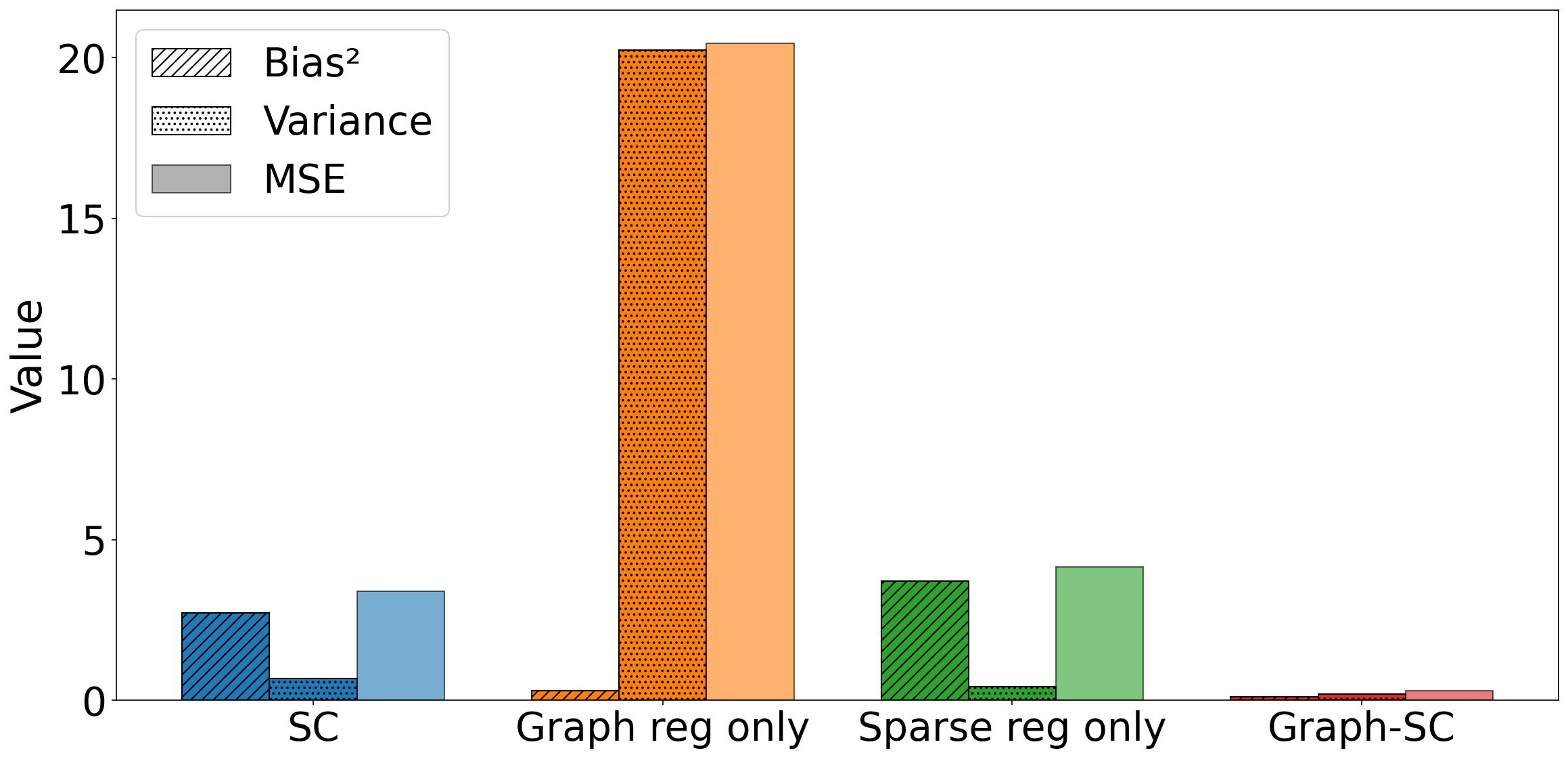}
        \textbf{(b)} Bias--variance decomposition
    \end{minipage}

    \caption{\textbf{Bias--Variance Verification.}
    \textbf{(a)} Distribution of estimated post-treatment effects across 50 repetitions. The dashed line represents the true average treatment effect on the treated (ATT$=2.0$).
    \textbf{(b)} Empirical decomposition of MSE into squared bias and variance. Graph-SC minimizes total error by correcting the heavy bias of Lasso methods while suppressing the variance of graph-only methods.}
    \label{fig:simulation_results}
\end{figure}

\subsection{California Proposition 99: Policy Benchmark}
We next evaluate Graph-SC on the canonical California Proposition 99 benchmark by Abadie et al.~\cite{abadie2010synthetic}, which estimates the effect of California's 1988 tobacco-control initiative on per-capita cigarette sales. This experiment serves two purposes. First, Graph-SC should preserve the established qualitative conclusion: cigarette sales in California declined substantially relative to a synthetic counterfactual after Proposition 99. Second, because this benchmark has a well-studied donor pool, it allows us to examine whether graph regularization changes the donor portfolio in a meaningful way rather than merely improving scalar fit. 

\noindent\textbf{Setup.}
Following the standard design, we use pre-treatment cigarette sales trajectories before 1988 to construct the synthetic counterfactual. For Graph-SC, the donor graph is built from similarity among pre-treatment donor trajectories. The graph penalty encourages donors with similar pre-treatment evolution to receive coherent weights, while the sparsity penalty prevents the estimator from spreading weight indiscriminately across the donor pool. We compare Graph-SC with classical SC\cite{abadie2003economic}, Synthetic-DiD\cite{arkhangelsky2021synthetic}, AugSC\cite{ben2021augmented}, Generalized-SC\cite{xu2017generalized}, Penalized-SCM\cite{abadie2021penalized}, Robust-SC\cite{amjad2018robust}, Matching+SC\cite{kellogg2021combining}, and two ablations: sparse-only and graph-only regularization.

A useful estimator in this benchmark should satisfy three desiderata simultaneously. First, it should achieve an accurate pre-treatment fit, since poor pre-treatment reconstruction weakens the credibility of the counterfactual. Second, it should preserve the canonical post-treatment discrepancy between California and its synthetic control. Third, it should avoid an overly concentrated donor portfolio, since a highly concentrated solution can be sensitive to idiosyncratic donor choices. We therefore report pre-treatment RMSPE, the post-treatment gap magnitude, the fraction of the SC post-treatment gap retained, the number of active donors, and donor-weight entropy. The retained gap is defined as

$\text{Gap retained}
=
\frac{
|\overline{g}_{post}^{\,m}|
}{
|\overline{g}_{post}^{\,SC}|
}
\times 100\%,
$
where $\overline{g}_{post}^{\,m}$ is the average post-treatment gap for method $m$. The entropy diagnostic is
$
H(w)=-\sum_j w_j \log(w_j+\epsilon),$
where larger values indicate a more diversified donor portfolio.

\textbf{Comparison Results.}
Figure~\ref{fig:prop99_placebo}(a) shows that both SC and Graph-SC closely track California before the 1988 intervention. After treatment, the observed California trajectory drops sharply below both synthetic counterfactuals, reproducing the canonical qualitative finding of Proposition 99. Importantly, Graph-SC does not eliminate the post-treatment discrepancy. Instead, it produces a slightly more conservative but still large post-treatment gap. This behavior is desirable: the graph penalty improves pre-treatment structure and donor diversification without collapsing the substantive post-treatment signal. Figure~\ref{fig:prop99_placebo}(b) reports an in-space placebo test. We reassign the 1988 intervention year to each untreated donor state and compute the corresponding placebo gap using Graph-SC. California's post-treatment gap is extreme relative to the placebo distribution, while its pre-treatment gap remains close to zero. We interpret this result as supporting the benchmark validity of Graph-SC rather than as a new identification claim beyond the standard assumptions of synthetic control.

\begin{figure}[t]
    \centering

    \begin{minipage}[t]{0.48\linewidth}
        \centering
        \includegraphics[width=\linewidth]{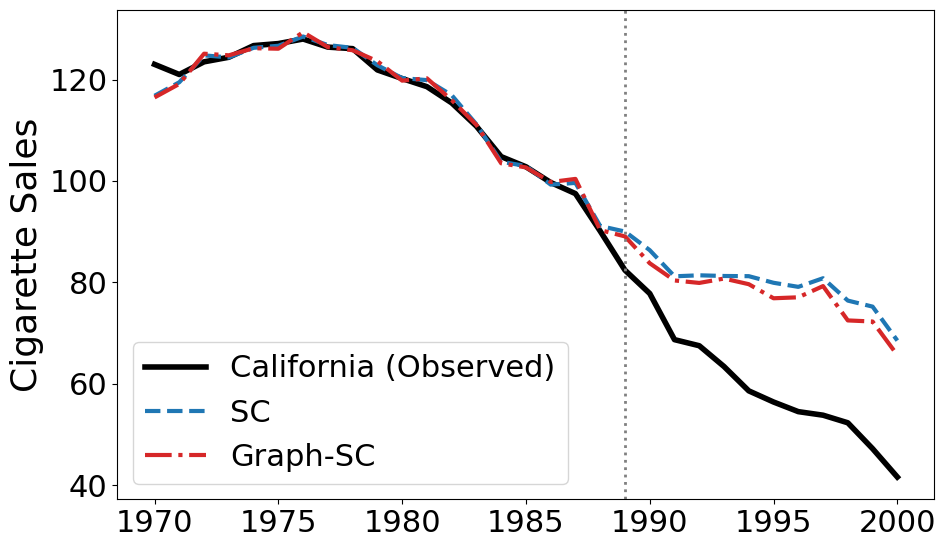}
        \vspace{-0.5em}
        \centerline{\small \textbf{(a)} Observed vs. synthetic trajectories}
    \end{minipage}
    \hfill
    \begin{minipage}[t]{0.48\linewidth}
        \centering
        \includegraphics[width=\linewidth]{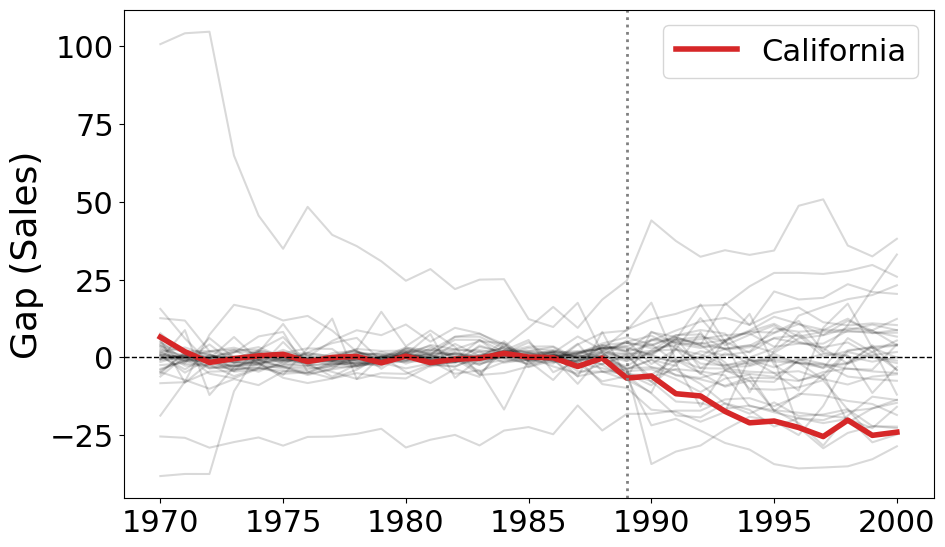}
        \vspace{-0.5em}
        \centerline{\small \textbf{(b)} In-space placebo test}
    \end{minipage}
    
    \caption{
    \textbf{Proposition 99 Benchmark.}
    \textbf{(a)} Observed vs. synthetic cigarette sales. Graph-SC closely tracks the pre-treatment trend while reproducing the post-treatment drop reported in \cite{abadie2010synthetic}.
    \textbf{(b)} In-space placebo test. The estimated gap for California is extreme compared to the distribution of placebo gaps generated by applying Graph-SC to the 38 untreated donor states, indicating statistical significance.
    }
    \label{fig:prop99_placebo}
\end{figure}

Table~\ref{tab:prop99_full_benchmark} shows that Graph-SC is the only method satisfying all three diagnostic criteria: very low pre-treatment error, at least $95\%$ retained post-treatment gap relative to SC, and high donor-weight entropy. Compared with classical SC, Graph-SC reduces pre-treatment RMSPE from $1.6564$ to $0.0893$, corresponding to a $94.6\%$ fit gain. At the same time, it preserves $95.4\%$ of the SC post-treatment gap, with $|\overline{g}_{post}|=18.6245$ compared with $19.5136$ for SC. This indicates that Graph-SC substantially improves pre-treatment reconstruction while retaining the main policy-relevant impact.

\begin{table*}[t]
\centering
\begin{threeparttable}
\caption{
\textbf{Proposition 99 benchmark comparison.}
Graph-SC is the only method that simultaneously achieves very low pre-treatment error, preserves at least $95\%$ of the canonical SC post-treatment gap, and produces a high-entropy donor portfolio. Exact-fit methods obtain near-zero pre-RMSPE but substantially attenuate the post-treatment gap, while high-gap alternatives have much worse pre-treatment fit or highly concentrated donor weights.
}
\label{tab:prop99_full_benchmark}

\scriptsize
\setlength{\tabcolsep}{3.5pt}
\renewcommand{\arraystretch}{1.12}

\begin{tabular}{lrrrrrrrr}
\toprule
\textbf{Method}
& \makecell{\textbf{Pre-RMSPE}\\\textbf{$\downarrow$}}
& \makecell{\textbf{Fit gain}\\\textbf{vs. SC $\uparrow$}}
& \makecell{\textbf{Post-RMSPE}\\\textbf{gap size}}
& \makecell{\textbf{$|\overline{g}_{post}|$}\\\textbf{gap magnitude}}
& \makecell{\textbf{Gap retained}\\\textbf{vs. SC $\uparrow$}}
& \makecell{\textbf{Active}\\\textbf{donors}}
& \makecell{\textbf{Weight}\\\textbf{entropy $\uparrow$}}
& \makecell{\textbf{Criteria}\\\textbf{met}} \\
\midrule

Sparse-only (Ablation)
& \crit{0.0058}
& 99.6\%
& 13.0945
& 12.3503
& 63.3\%
& 20
& 2.5791
& 1/3 \\

Graph-only (Ablation)
& \crit{0.0663}
& 96.0\%
& 18.4235
& 17.3714
& 89.0\%
& 38
& \crit{3.3940}
& 2/3 \\

\oursrow
\textbf{Graph-SC (Ours)}
& \textbf{0.0893}
& 94.6\%
& 19.8918
& 18.6245
& 95.4\%
& 38
& \textbf{3.4112}
& \textbf{3/3} \\

\midrule

SC \cite{abadie2003economic}
& 1.6564
& 0.0\%
& 20.6056
& 19.5136
& \crit{100.0\%}
& 6
& 1.4752
& 1/3 \\

AugSC \cite{ben2021augmented}
& \crit{0.0001}
& 100.0\%
& 13.1322
& 12.3747
& 63.4\%
& 6
& 1.4752
& 1/3 \\

Synthetic-DiD \cite{arkhangelsky2021synthetic}
& 1.6564
& 0.0\%
& 18.8482
& 17.6478
& 90.4\%
& 6
& 1.4752
& 0/3 \\

Generalized-SC \cite{xu2017generalized}
& 1.3718
& 17.2\%
& 7.3630
& 6.4959
& 33.3\%
& 0
& 0.0000
& 0/3 \\

Penalized-SCM \cite{abadie2021penalized}
& 1.6564
& 0.0\%
& 20.6056
& 19.5136
& \crit{100.0\%}
& 6
& 1.4752
& 1/3 \\

Robust-SC \cite{amjad2018robust}
& 2.4607
& $-48.6\%$
& 19.0547
& 17.3461
& 88.9\%
& 38
& \crit{3.2471}
& 1/3 \\

Matching+SC \cite{kellogg2021combining}
& 3.8165
& $-130.4\%$
& 27.0390
& 25.2279
& \crit{129.3\%}
& 4
& 1.2661
& 1/3 \\

\bottomrule
\end{tabular}

\begin{tablenotes}
\footnotesize
\item Note: Green cells indicate that a method satisfies one of three diagnostic criteria: pre-RMSPE $\leq 0.1$, retained post-treatment gap $\geq 95\%$ of SC, or donor-weight entropy $\geq 3.0$. The final column counts how many of these criteria are satisfied. These thresholds are used only as descriptive diagnostics and are not used for training or hyperparameter selection. The Graph-SC row is highlighted in yellow. Post-treatment RMSPE and $|\overline{g}_{post}|$ measure the magnitude of the post-treatment discrepancy rather than supervised prediction error.
\end{tablenotes}

\end{threeparttable}
\end{table*}

The comparison also clarifies why low pre-treatment error alone is not sufficient. AugSC and the sparse-only ablation achieve nearly exact pre-treatment fit, with pre-RMSPE values of $0.0001$ and $0.0058$, respectively. However, both substantially attenuate the post-treatment gap, retaining only $63.4\%$ and $63.3\%$ of the SC gap. These results suggest that exact pre-treatment interpolation can over-adjust the counterfactual and weaken the post-treatment signal. Conversely, Matching+SC produces the largest post-treatment discrepancy, retaining $129.3\%$ of the SC gap, but it has much worse pre-treatment fit with pre-RMSPE $3.8165$ and a concentrated donor portfolio with entropy $1.2661$. Robust-SC has high entropy but also much worse pre-treatment fit than Graph-SC and retains only $88.9\%$ of the SC gap. Thus, the apparent alternatives each fail on at least one diagnostic axis.

\textbf{Role of the graph penalty.}
The ablation results show that the full Graph-SC objective is not simply a cosmetic combination of sparsity and graph regularization. Sparse-only regularization achieves an excellent fit but severely attenuates the post-treatment gap. Graph-only regularization produces a high-entropy donor portfolio but retains only $89.0\%$ of the SC gap. The full Graph-SC estimator balances these effects: it maintains very low pre-treatment RMSPE, preserves the substantive post-treatment gap, and produces the highest donor-weight entropy in the table. This supports the central claim of the paper: graph regularization is most useful when evaluated jointly through fit, signal preservation, and donor-weight stability rather than through a single scalar error metric.

\subsection{Waymo Open Motion Dataset: Observational Counterfactual Contrasts}

We apply Graph-SC to a large-scale safety analysis using the Waymo Open Motion Dataset (WOMD). Prior studies using WOMD have examined scenarios such as human drivers following AVs, AVs following human-driven vehicles, and lane-changing behaviors involving AVs \cite{Chung2025, Zhang2025}. However, a closely related yet distinct question in mixed autonomy traffic remains largely unexplored. Specifically, it is still unclear whether the mere presence of an AV in the surrounding traffic influences the car following behavior of nearby human drivers, particularly given the limited penetration of AVs in real-world traffic.

We analyze Human-Human car-following pairs where an AV becomes visible to the follower ($t \ge 5.0$s, range $<20$m), testing the hypothesis that the presence of AVs acts as a traffic calming stimulus, potentially because human drivers may adopt more conservative spacing to allocate attention to the novel stimulus (the AV) rather than purely forward monitoring. To demonstrate robustness across domains, we apply the same hyperparameters used in the benchmark study ($\lambda_1=0.05, \lambda_g=10.0$). We quantify safety using headway standard deviation (\texttt{headway\_std}), a proxy for flow instability, and minimum Time-to-Collision (\texttt{min\_TTC}), a measure of immediate crash risk. The detailed extraction pipeline of the dataset is provided in Appendix~\ref{app:data_and_setup}. 

\begin{table*}[t]
\centering
\begin{threeparttable}
\caption{
\textbf{WOMD multi-baseline robustness comparison.}
Graph-SC is evaluated on 50 AV-exposed car-following events with 300 candidate donors per event.
For each outcome, green cells indicate that a method satisfies one of three diagnostic criteria:
low pre-exposure error, large post-exposure signal relative to the best pre-fit-qualified method, or controlled placebo discrepancy.
For headway variability, Graph-SC is the only method satisfying all three criteria, achieving near-exact pre-exposure fit while producing the largest post-exposure signal among well-fitting methods.
For minimum TTC, Graph-SC and AugSC both satisfy all three criteria; Graph-SC gives a slightly smaller main gap but also the lowest placebo discrepancy and lower cross-unit dispersion.
}
\label{tab:womd_full_benchmark}

\scriptsize
\setlength{\tabcolsep}{3.2pt}
\renewcommand{\arraystretch}{1.12}

\begin{tabular}{llrrrrrrrr}
\toprule
\textbf{Outcome}
& \textbf{Method}
& \makecell{\textbf{Pre-RMSPE}\textbf{$\downarrow$}}
& \makecell{\textbf{Fit gain}\\\textbf{vs. SC $\uparrow$}}
& \makecell{\textbf{$|\overline{g}_{post}|$}\\\textbf{main gap $\uparrow$}}
& \makecell{\textbf{Gap retained}\\\textbf{vs. best-fit $\uparrow$}}
& \makecell{\textbf{$|\overline{g}_{placebo}|$}\\\textbf{placebo $\downarrow$}}
& \makecell{\textbf{Criteria met}} \\
\midrule

\multirow{8}{*}{\texttt{headway std}}

& SC \cite{abadie2003economic}
& 1.1738
& 0.0\%
& 5.3654
& 82.0\%
& \crit{1.2637}
& 1/3 \\

& AugSC \cite{ben2021augmented}
& \crit{\textbf{0.0148}}
& 98.7\%
& 6.0753
& 92.9\%
& \crit{\textbf{1.2366}}
& 2/3 \\

& Synthetic-DiD \cite{arkhangelsky2021synthetic}
& 1.1738
& 0.0\%
& 5.3654
& 82.0\%
& \crit{1.2637}
& 1/3 \\

& Generalized-SC \cite{xu2017generalized}
& 1.6611
& $-41.5\%$
& 12.2329
& \crit{187.0\%}
& 2.3937
& 1/3 \\

& Robust-SC \cite{amjad2018robust}
& 2.0516
& $-74.8\%$
& 34.0756
& \crit{520.9\%}
& 3.1861
& 1/3 \\

& Penalized-SCM \cite{abadie2021penalized}
& 1.2156
& $-3.6\%$
& 5.5137
& 84.3\%
& \crit{1.1103}
& 1/3 \\

& Matching+SC \cite{kellogg2021combining}
& 1.5504
& $-32.1\%$
& 5.0083
& 76.6\%
& 1.3839
& 0/3 \\

\oursrow
& \textbf{Graph-SC (Ours)}
& \textbf{0.0236}
& 98.0\%
& 6.5412
& \textbf{100.0\%}
& \textbf{1.3076}
& \textbf{3/3} \\

\midrule

\multirow{8}{*}{\texttt{min TTC}}

& SC \cite{abadie2003economic}
& 0.7203
& 0.0\%
& 1.7767
& 75.1\%
& \crit{0.9282}
& 1/3 \\

& AugSC \cite{ben2021augmented}
& \crit{\textbf{0.0385}}
& 94.7\%
& 2.3644
& \crit{\textbf{100.0\%}}
& \crit{1.2018}
& \textbf{3/3} \\

& Synthetic-DiD \cite{arkhangelsky2021synthetic}
& 0.7203
& 0.0\%
& 1.7339
& 73.3\%
& \crit{0.9919}
& 1/3 \\

& Generalized-SC \cite{xu2017generalized}
& 1.2394
& $-72.1\%$
& 2.3770
& \crit{100.5\%}
& 2.4929
& 1/3 \\

& Robust-SC \cite{amjad2018robust}
& 1.3348
& $-85.3\%$
& 2.7689
& \crit{117.1\%}
& 1.3053
& 1/3 \\

& Penalized-SCM \cite{abadie2021penalized}
& 0.7331
& $-1.8\%$
& 1.7563
& 74.3\%
& \crit{\textbf{0.8620}}
& 1/3 \\

& Matching+SC \cite{kellogg2021combining}
& 0.9703
& $-34.7\%$
& 2.2626
& \crit{95.7\%}
& \crit{0.9306}
& 2/3 \\

\oursrow
& \textbf{Graph-SC (Ours)}
& \textbf{0.0440}
& 93.9\%
& 2.3171
& \textbf{98.0\%}
& \textbf{1.1786}
& \textbf{3/3} \\

\bottomrule
\end{tabular}

\begin{tablenotes}
\footnotesize
\item Note: Green cells indicate that a method satisfies one of three outcome-specific diagnostic criteria:
pre-RMSPE $\leq 0.1$;
retained main gap $\geq 95\%$ of the largest main gap among pre-fit-qualified methods;
or placebo gap within $110\%$ of the smallest placebo gap among pre-fit-qualified methods.
The final column counts how many criteria are satisfied.
These thresholds are descriptive diagnostics and are not used for training or hyperparameter selection.
The Graph-SC rows are highlighted in yellow.
Large main gaps from poorly fitting methods should not be interpreted as better counterfactual estimates, because they may reflect pre-exposure mismatch rather than post-exposure signal.
\end{tablenotes}

\end{threeparttable}
\end{table*}

Table~\ref{tab:womd_full_benchmark} summarizes the WOMD multi-baseline comparison. For headway, Graph-SC is the only method satisfying all three diagnostic criteria. Compared with standard SC, Graph-SC reduces pre-exposure RMSPE from $1.1738$ to $0.0236$, corresponding to a $98.0\%$ fit gain, while increasing the average absolute post-exposure gap from $5.3654$ to $6.5412$. Compared with AugSC, which has the lowest pre-RMSPE, Graph-SC produces a larger post-exposure signal ($6.5412$ vs. $6.0753$). This result is important because it shows that Graph-SC does not merely improve pre-treatment fit; it selects a graph-coherent counterfactual that preserves a stronger post-exposure contrast.

For min TTC, the comparison is more conservative. Graph-SC and AugSC both satisfy all three criteria. AugSC has slightly lower pre-RMSPE and a slightly larger main gap, whereas Graph-SC has the smallest placebo discrepancy among the pre-fit-qualified methods and slightly lower cross-unit dispersion. Specifically, Graph-SC achieves pre-RMSPE $0.0440$, retains $98.0\%$ of the strongest pre-fit-qualified TTC signal, and has the lowest qualified placebo gap of $1.1786$. This suggests that Graph-SC preserves the immediate-risk dynamics estimated by the strongest exact-fit baseline while avoiding an increase in placebo discrepancy.

\section{Discussion}
This paper studies synthetic control under donor redundancy: settings where many donor portfolios achieve similar pre-intervention fit but yield unstable weights and potentially different post-intervention counterfactuals. Graph-SC addresses this instability by introducing a donor similarity graph learned from pre-intervention data and regularizing the synthetic-control weights through a Laplacian penalty. Across the controlled simulation, the Proposition 99 benchmark, and WOMD observational counterfactual analyses, the main empirical pattern is consistent with the theory: Graph-SC is most useful when pre-treatment fit alone is insufficient to distinguish among many plausible donor mixtures. In this regime, graph regularization improves donor-weight coherence and robustness while preserving the interpretability of weight-based counterfactual construction.

Several practical choices influence Graph-SC performance. The donor graph depends on the pre-intervention feature representation, similarity metric, neighborhood size, and whether edges are weighted or unweighted. Although our ablations in appendix~\ref{app:simulation_graph_ablation} suggest that the method is reasonably stable across several graph choices, graph learning remains an important modeling decision. The diagnostic thresholds used in our tables are intended to summarize tradeoffs among fit, signal preservation, placebo behavior, and donor-weight stability. Promising directions for future work include scalable graph construction via approximate nearest neighbors, richer graph priors such as weighted or multi-view graphs, and joint graph weight learning, extensions to time-varying weights and multi-outcome panels, and uncertainty quantification tailored to graph-regularized counterfactuals.



\newpage
\bibliographystyle{unsrt}
\small
\bibliography{neurips_2026}

@article{abadie2003economic,
  title={The economic costs of conflict: A case study of the Basque Country},
  author={Abadie, Alberto and Gardeazabal, Javier},
  journal={American economic review},
  volume={93},
  number={1},
  pages={113--132},
  year={2003},
  publisher={American Economic Association}
}

@article{abadie2010synthetic,
  title={Synthetic control methods for comparative case studies: Estimating the effect of California’s tobacco control program},
  author={Abadie, Alberto and Diamond, Alexis and Hainmueller, Jens},
  journal={Journal of the American statistical Association},
  volume={105},
  number={490},
  pages={493--505},
  year={2010},
  publisher={Taylor \& Francis}
}

@article{li2008network,
  title={Network-constrained regularization and variable selection for analysis of genomic data},
  author={Li, Caiyan and Li, Hongzhe},
  journal={Bioinformatics},
  volume={24},
  number={9},
  pages={1175--1182},
  year={2008},
  publisher={Oxford University Press}
}

@article{zhang2025,
  title={Can the Waymo Open Motion Dataset Support Realistic Behavioral Modeling? A Validation Study with Naturalistic Trajectories},
  author={Zhang, Yanlin and Chung, Sungyong and Li, Nachuan and Monzer, Dana and Mahmassani, Hani S and Hamdar, Samer H and Talebpour, Alireza},
  journal={arXiv preprint arXiv:2509.03515},
  year={2025}
}

@article{chung2025,
  title={Characterizing Lane-Changing Behavior in Mixed Traffic},
  author={Chung, Sungyong and Talebpour, Alireza and Hamdar, Samer H},
  journal={arXiv preprint arXiv:2512.07219},
  year={2025}
}

@InProceedings{Kan_2024_icra, author={Chen, Kan and Ge, Runzhou and Qiu, Hang and Ai-Rfou, Rami and Qi, Charles R. and Zhou, Xuanyu and Yang, Zoey and Ettinger, Scott and Sun, Pei and Leng, Zhaoqi and Mustafa, Mustafa and Bogun, Ivan and Wang, Weiyue and Tan, Mingxing and Anguelov, Dragomir}, title={WOMD-LiDAR: Raw Sensor Dataset Benchmark for Motion Forecasting}, month={May}, 
booktitle= {Proceedings of the IEEE International Conference on Robotics and Automation (ICRA)}, 
year={2024} }

@inproceedings{ettinger2021waymo,
  title={Large scale interactive motion forecasting for autonomous driving: The waymo open motion dataset},
  author={Ettinger, Scott and Cheng, Shuyang and Caine, Benjamin and Liu, Chenxi and Zhao, Hang and Pradhan, Sabeek and Chai, Yuning and others},
  booktitle={Proceedings of the IEEE/CVF International Conference on Computer Vision},
  pages={9710--9719},
  year={2021}
}

@article{beck2009fast,
  title={A fast iterative shrinkage-thresholding algorithm for linear inverse problems},
  author={Beck, Amir and Teboulle, Marc},
  journal={SIAM journal on imaging sciences},
  volume={2},
  number={1},
  pages={183--202},
  year={2009},
  publisher={SIAM}
}

@article{athey2017state,
  title={The state of applied econometrics: Causality and policy evaluation},
  author={Athey, Susan and Imbens, Guido W},
  journal={Journal of Economic perspectives},
  volume={31},
  number={2},
  pages={3--32},
  year={2017},
  publisher={American Economic Association 2014 Broadway, Suite 305, Nashville, TN 37203-2418}
}

@article{ben2021augmented,
  title={The augmented synthetic control method},
  author={Ben-Michael, Eli and Feller, Avi and Rothstein, Jesse},
  journal={Journal of the American Statistical Association},
  volume={116},
  number={536},
  pages={1789--1803},
  year={2021},
  publisher={Taylor \& Francis}
}

@article{arkhangelsky2021synthetic,
  title={Synthetic difference-in-differences},
  author={Arkhangelsky, Dmitry and Athey, Susan and Hirshberg, David A and Imbens, Guido W and Wager, Stefan},
  journal={American Economic Review},
  volume={111},
  number={12},
  pages={4088--4118},
  year={2021},
  publisher={American Economic Association 2014 Broadway, Suite 305, Nashville, TN 37203}
}

@article{xu2017generalized,
  title={Generalized synthetic control method: Causal inference with interactive fixed effects models},
  author={Xu, Yiqing},
  journal={Political Analysis},
  volume={25},
  number={1},
  pages={57--76},
  year={2017},
  publisher={Cambridge University Press}
}

@article{amjad2018robust,
  title={Robust synthetic control},
  author={Amjad, Muhammad and Shah, Devavrat and Shen, Dennis},
  journal={Journal of Machine Learning Research},
  volume={19},
  number={22},
  pages={1--51},
  year={2018}
}

@article{abadie2021penalized,
  title={A penalized synthetic control estimator for disaggregated data},
  author={Abadie, Alberto and L’hour, J{\'e}r{\'e}my},
  journal={Journal of the American Statistical Association},
  volume={116},
  number={536},
  pages={1817--1834},
  year={2021},
  publisher={Taylor \& Francis}
}

@inproceedings{hallac2015network,
  title={Network lasso: Clustering and optimization in large graphs},
  author={Hallac, David and Leskovec, Jure and Boyd, Stephen},
  booktitle={Proceedings of the 21th ACM SIGKDD international conference on knowledge discovery and data mining},
  pages={387--396},
  year={2015}
}

@article{kellogg2021combining,
  title={Combining matching and synthetic control to tradeoff biases from extrapolation and interpolation},
  author={Kellogg, Maxwell and Mogstad, Magne and Pouliot, Guillaume A and Torgovitsky, Alexander},
  journal={Journal of the American statistical association},
  volume={116},
  number={536},
  pages={1804--1816},
  year={2021},
  publisher={Taylor \& Francis}
}

\newpage
\appendix

\section{Proofs}

\subsection{Proof of Proposition \ref{prop:nonunique}}
\begin{proof}
\label{proof:nonunique}
Feasibility: since $\mathbf{1}^\top v=0$, we have $\mathbf{1}^\top(\widehat w+t v)=\mathbf{1}^\top \widehat w=1$. Nonnegativity holds by construction for sufficiently small $t$ (because $\widehat w$ is feasible). Objective value:
\[
x_1-X_0(\widehat w+t v)=(x_1-X_0\widehat w)-t(X_0 v)=(x_1-X_0\widehat w),
\]
because $X_0 v=0$. Squaring gives equality of losses. Thus, multiple feasible minimizers exist.
\end{proof}

\subsection{Proof of the complexity of algorithm~\ref{alg:graph_learning}}
\label{proof:graph_complexity}
Computing all pairwise similarities costs $O(J^2\cdot c_s)$ time, where $c_s$ is the cost of evaluating $s(\cdot,\cdot)$. For common similarities with dense features, $c_s=O(k)$, so the dominant cost is $O(J^2 k)$. Selecting $k_{\mathrm{nn}}$ nearest neighbors for each donor can be implemented in $O(J)$ expected time per row using linear-time selection, giving $O(J^2)$ expected time overall, or in $O(J\log k_{\mathrm{nn}})$ per row using a heap, giving $O(J^2\log k_{\mathrm{nn}})$ overall. When the adjacency is stored sparsely, symmetrization and degree computation cost $O(|E|)$ with $|E|=O(Jk_{\mathrm{nn}})$, and subsequent Laplacian-vector products $Lw$ also cost $O(|E|)$. Memory is $O(J^2)$ if the full similarity matrix is stored, but can be reduced to $O(Jk_{\mathrm{nn}})$ by computing similarities row-by-row and keeping only the top-$k_{\mathrm{nn}}$ candidates per node.

\subsection{Proof of Lemma~\ref{lem:laplacian_identity} }
\begin{proof}
\label{proof:laplacian_identity}
Using $L=I-D^{-1/2}AD^{-1/2}$,
\[
w^\top L w
=
\|w\|_2^2 - w^\top D^{-1/2}AD^{-1/2}w.
\]
Let $\tilde w:=D^{-1/2}w$. Then $w^\top D^{-1/2}AD^{-1/2}w=\tilde w^\top A \tilde w=\sum_{u,v}A_{uv}\tilde w_u\tilde w_v$ and $\|w\|_2^2=\sum_u d_u\tilde w_u^2$ since $w_u=\sqrt{d_u}\tilde w_u$. Therefore,
\begin{align*}
w^\top L w
&=
\sum_u d_u\tilde w_u^2 - \sum_{u,v}A_{uv}\tilde w_u\tilde w_v \\
&=
\frac12\sum_{u,v}A_{uv}\big(\tilde w_u^2+\tilde w_v^2-2\tilde w_u\tilde w_v\big)
=
\frac12\sum_{u,v}A_{uv}(\tilde w_u-\tilde w_v)^2,
\end{align*}
which is exactly the stated identity after substituting $\tilde w_u=w_u/\sqrt{d_u}$.
\end{proof}

\subsection{Proof of Theorem~\ref{thm:bias_variance_tradeoff}}
\begin{proof}
\label{proof:bias_variance_tradeoff}
The decomposition $\widehat w-w^\star=(\widehat w-\bar w)+(\bar w-w^\star)$.


\paragraph{Bias upper bound.}
Let $\phi(w)$ be the smooth part of the population objective:
\[
\phi(w)=\frac12\|X_0w^\star-X_0w\|_2^2+\frac{\lambda_g}{2}w^\top Lw+\frac{\lambda_{sum}}{2}(\mathbf{1}^\top w-1)^2.
\]
Then $\phi$ is $\mu$-strongly convex.
Optimality at $\bar w$ gives
\[
0 \in \nabla\phi(\bar w) + \lambda_1 \partial\| \bar w\|_1 + N_{\mathbb{R}^J_+}(\bar w),
\]
where $N_{\mathbb{R}^J_+}$ is the normal cone.
Let $\xi\in\partial\|\bar w\|_1$ and $\nu\in N_{\mathbb{R}^J_+}(\bar w)$ satisfy
$\nabla\phi(\bar w)+\lambda_1\xi+\nu=0$.
Also note that $w^\star$ is feasible ($w^\star\ge 0$), so
$\nu^\top(\bar w-w^\star)\ge 0$ by normal cone properties.

Strong convexity implies the inequality.
\[
\mu\|\bar w-w^\star\|_2
\le \|\nabla\phi(w^\star)-\nabla\phi(\bar w)\|_2
=
\|\nabla\phi(w^\star)+\lambda_1\xi+\nu\|_2
\le \|\nabla\phi(w^\star)\|_2+\lambda_1\|\xi\|_2+\|\nu\|_2,
\]
and using $\nu^\top(\bar w-w^\star)\ge 0$ one can drop $\nu$ in a one-sided bound.
Compute
\[
\nabla\phi(w)=X_0^\top(X_0w-X_0w^\star)+\lambda_g Lw+\lambda_{sum}(\mathbf{1}^\top w-1)\mathbf{1}.
\]
Thus at $w^\star$,
\[
\nabla\phi(w^\star)=\lambda_g Lw^\star + \lambda_{sum}(\mathbf{1}^\top w^\star-1)\mathbf{1}.
\]
Combine and use $\|\mathbf{1}\|_2=\sqrt{J}$ to obtain \eqref{eq:bias_upper}.

\paragraph{bias lower bound.}
Set $\lambda_1=0$ and assume $\mathbf{1}^\top w^\star=1$.
Ignoring the inequality constraint,
the population objective is differentiable, and the minimizer satisfies
\[
(X_0^\top X_0 + Q)\bar w = X_0^\top X_0 w^\star + \lambda_{sum} \mathbf{1}.
\]
Subtract $(X_0^\top X_0 + Q)w^\star$ from both sides to get
\[
(X_0^\top X_0 + Q)(\bar w-w^\star)= -\lambda_g Lw^\star,
\]
hence
\[
\bar w-w^\star = -(X_0^\top X_0 + Q)^{-1}\lambda_g Lw^\star.
\]
Project both sides onto $\mathbf{1}^\perp$; note that $Q$ acts as $\lambda_gL$ on $\mathbf{1}^\perp$
and $\|(X_0^\top X_0+Q)^{-1}\|_{\mathrm{op}}=1/\lambda_{\min}(X_0^\top X_0+Q)$.
Then
\[
\begin{aligned}
\|\bar w-w^\star\|_2
&\ge
\|P_{\mathbf{1}^\perp}(\bar w-w^\star)\|_2 \\
&=
\lambda_g\|(X_0^\top X_0+Q)^{-1} P_{\mathbf{1}^\perp}(Lw^\star)\|_2\\
&\ge
\frac{\lambda_g}{\|X_0^\top X_0+Q\|_{\mathrm{op}}}\|P_{\mathbf{1}^\perp}(Lw^\star)\|_2,
\end{aligned}
\]

which is \eqref{eq:bias_lower}.
\end{proof}

\subsection{Proofs for the optimal convergence rate of solving Graph-SC}
\label{proof:conv_gsc_compact}

\paragraph{Preliminaries and notation}
We study $F(w)=f(w)+g(w)$ where $f$ is convex and differentiable with $L_f$-Lipschitz gradient, and
$g$ is proper, closed, convex with an efficiently computable proximal map. The proximal operator is
\[
\mathrm{prox}_{\eta g}(v)\in\arg\min_{u\in\mathbb{R}^J}\left\{\frac12\|u-v\|_2^2+\eta g(u)\right\}.
\]

\paragraph{Smoothness of the GSC loss and the constant $L_f$}
\label{app:smoothness}

\begin{lemma}[Gradient and Hessian of $f$]
\label{lem:grad_hess_full}
For $f$ defined in \eqref{eq:fg_def},
\[
\nabla f(w)=X_0^\top(X_0w-x_1)+\lambda_g Lw+\lambda_{sum}(\mathbf{1}^\top w-1)\mathbf{1},
\qquad
\nabla^2 f = X_0^\top X_0+\lambda_g L+\lambda_{sum}\mathbf{1}\mathbf{1}^\top.
\]
\end{lemma}

\begin{proof}
Differentiate each term:
\[
\nabla \tfrac12\|x_1-X_0w\|^2 = -X_0^\top(x_1-X_0w)=X_0^\top(X_0w-x_1),
\]
\[
\nabla \tfrac{\lambda_g}{2}w^\top Lw = \lambda_g Lw \quad (\text{since }L=L^\top),
\]
\[
\nabla \tfrac{\lambda_{sum}}{2}(\mathbf{1}^\top w-1)^2 = \lambda_{sum}(\mathbf{1}^\top w-1)\mathbf{1}.
\]
Since $f$ is quadratic, differentiating again yields the constant Hessian above.
\end{proof}

\begin{lemma}[Lipschitz gradient]
\label{lem:lipschitz_full}
Let $A\triangleq X_0^\top X_0+\lambda_g L+\lambda_{sum}\mathbf{1}\mathbf{1}^\top$.
Then $\nabla f$ is $L_f$-Lipschitz with $L_f=\|A\|_{\mathrm{op}}=\lambda_{\max}(A)$, and
\[
L_f\le \|X_0\|_{\mathrm{op}}^2+\lambda_g\lambda_{\max}(L)+\lambda_{sum} J.
\]
\end{lemma}

\begin{proof}
By Lemma~\ref{lem:grad_hess_full}, $\nabla f(w)=Aw-b$ for $b=X_0^\top x_1+\lambda_{sum}\mathbf{1}$.
Thus for any $u,v$,
\[
\|\nabla f(u)-\nabla f(v)\|_2=\|A(u-v)\|_2\le \|A\|_{\mathrm{op}}\|u-v\|_2,
\]
so $L_f=\|A\|_{\mathrm{op}}=\lambda_{\max}(A)$. For the upper bound, use
$\lambda_{\max}(B+C)\le \lambda_{\max}(B)+\lambda_{\max}(C)$ and
$\lambda_{\max}(\mathbf{1}\mathbf{1}^\top)=\|\mathbf{1}\|_2^2=J$, giving
\[
\lambda_{\max}(A)\le \lambda_{\max}(X_0^\top X_0)+\lambda_2\lambda_{\max}(L)+\lambda_{sum} J
=\|X_0\|_{\mathrm{op}}^2+\lambda_g\lambda_{\max}(L)+\lambda_{sum} J.
\]
\end{proof}

\paragraph{Closed-form proximal map for $g$}
\label{app:prox}

\begin{lemma}
\label{lem:prox_full}
Let $g(w)=\lambda_1\|w\|_1$. For any $\eta>0$,
\[
\mathrm{prox}_{\eta g}(v)=(v-\eta\lambda_1\mathbf{1})_+,
\]
i.e., $[\mathrm{prox}_{\eta g}(v)]_j=\max\{0,v_j-\eta\lambda_1\}$.
\end{lemma}

\begin{proof}
The proximal problem is
\[
\min_{u\ge 0}\ \frac12\|u-v\|_2^2+\eta\lambda_1\|u\|_1.
\]
It is separable. For each coordinate $j$ we solve
\[
\min_{u_j\ge 0}\ \frac12(u_j-v_j)^2+\eta\lambda_1|u_j|.
\]
On $u_j\ge 0$, $|u_j|=u_j$. The unconstrained minimizer satisfies $(u_j-v_j)+\eta\lambda_1=0$,
hence $u_j=v_j-\eta\lambda_1$. Enforcing $u_j\ge 0$ yields $u_j^\star=\max\{0,v_j-\eta\lambda_1\}$.
Stacking coordinates proves the claim.
\end{proof}

\paragraph{Convergence of Algorithm~\ref{alg:fista_gsc_compact}}
\label{app:fista_rate}

By Lemma~\ref{lem:lipschitz_full}, $f$ is convex with $L_f$-Lipschitz gradient.
By Lemma~\ref{lem:prox_full}, the proximal operator of $g$ is well-defined and computable.
Therefore, $F=f+g$ satisfies the assumptions of FISTA \cite{beck2009fast}.

\begin{theorem}[FISTA rate specialized to GSC]
\label{thm:fista_full}
Let $F=f+g$ be the GSC objective, and let $L_f$ be as in
Lemma~\ref{lem:lipschitz_full}. Run Algorithm~\ref{alg:fista_gsc_compact} with $\eta=1/L_f$.
Then for any minimizer $w^\star$,
\[
F(w^{k})-F(w^\star)\le \frac{2L_f\|w^0-w^\star\|_2^2}{(k+1)^2},\qquad k\ge 1.
\]
\end{theorem}

\begin{proof}
We provide a full proof by instantiating the standard estimate-sequence argument for FISTA.

We define
\begin{align}
\label{eq:fg_def}
f(w)
&=
\frac12\|x_1-X_0w\|_2^2
+\frac{\lambda_g}{2}w^\top Lw
+\frac{\lambda_{sum}}{2}(\mathbf{1}^\top w-1)^2,\\
g(w)
&=
\lambda_1\|w\|_1,
\end{align}
where $L\succeq 0$ is the Laplacian of an undirected, unweighted similarity graph.

Each iteration only needs products
$X_0w$, $X_0^\top v$, and $Lw$ outlined as algorithm~\ref{alg:fista_gsc_compact}.

\begin{algorithm}
\caption{FISTA for Graph Synthetic Control}
\label{alg:fista_gsc_compact}
\begin{algorithmic}[1]
\Require $x_1, X_0, L$, parameters $\lambda_1,\lambda_g,\lambda_{\mathrm{sum}}$, step size $\eta \le 1/L_f$
\State Initialize $w^{0}\in\mathbb{R}^J_+$, set $z^{0}\gets w^{0}$ and $t_0\gets 1$
\For{$k=0,1,2,\ldots$}
    \State $u^k \gets z^k - \eta\big[
    X_0^\top(X_0 z^k-x_1)
    + \lambda_g L z^k
    + \lambda_{\mathrm{sum}}(\mathbf{1}^\top z^k-1)\mathbf{1}
    \big]$
    \State $w^{k+1} \gets \big(u^k-\eta\lambda_1\mathbf{1}\big)_{+}$
    \State $t_{k+1}\gets \dfrac{1+\sqrt{1+4t_k^2}}{2}$
    \State $z^{k+1}\gets w^{k+1} + \dfrac{t_k-1}{t_{k+1}}(w^{k+1}-w^{k})$
\EndFor
\end{algorithmic}
\end{algorithm}

\begin{theorem}[Convergence of first-order solvers for Graph-SC]
\label{thm:conv_gsc_compact}

Define
\begin{equation}
\label{eq:Lmu_defs_compact}
L_f \triangleq \lambda_{\max}\!\Big(X_0^\top X_0+\lambda_gL+\lambda_{sum}\mathbf{1}\mathbf{1}^\top\Big),
\end{equation}
\begin{equation}
\label{eq:mu_defs_compact}
\mu_f \triangleq \lambda_{\min}\!\Big(X_0^\top X_0+\lambda_gL+\lambda_{sum}\mathbf{1}\mathbf{1}^\top\Big).
\end{equation}

Let $F=f+g$ be defined in \eqref{eq:fg_def}, and choose $\eta=1/L_f$ with $L_f$ in \eqref{eq:Lmu_defs_compact}.
Then:
\begin{enumerate}
\item (\emph{Convex}) The iterates of Algorithm~\ref{alg:fista_gsc_compact} satisfy
\[
F(w^{k})-F(w^\star)\;\le\;\frac{2L_f\|w^0-w^\star\|_2^2}{(k+1)^2},\qquad k\ge 1,
\]
for any minimizer $w^\star$.
\item (\emph{Strongly convex}) If $\mu_f>0$, then the restarted variant of
Algorithm~\ref{alg:fista_gsc_compact} (periodic restart) converges linearly:
\[
F(w^{k})-F(w^\star)\;\le\;C\Big(1-\sqrt{\mu_f/L_f}\Big)^{k}\!,
\]
for a constant $C$ depending on initialization.
\end{enumerate}
\end{theorem}

\begin{proof}
Define the quadratic upper model of $f$ at a point $y$:
\[
Q_{L_f}(x,y)\triangleq f(y)+\langle \nabla f(y),x-y\rangle+\frac{L_f}{2}\|x-y\|_2^2+g(x).
\]
By $L_f$-smoothness (Lemma~\ref{lem:lipschitz_full}), for all $x,y$,
\[
F(x)=f(x)+g(x)\le f(y)+\langle \nabla f(y),x-y\rangle+\frac{L_f}{2}\|x-y\|^2+g(x)=Q_{L_f}(x,y).
\]
Algorithm~\ref{alg:fista_gsc_compact} performs
\[
w^{k+1}=\arg\min_x Q_{L_f}(x,z^k),
\]
and the extrapolation parameters satisfy $t_{k+1}^2-t_{k+1}=t_k^2$.

Following \cite{beck2009fast}, define
\[
\Phi_k(x)\triangleq t_k Q_{L_f}(x,z^k) + \frac{L_f}{2}\|x-w^0\|_2^2.
\]
One can show that the choice of $t_k$ and the update of $z^k$
ensure
\[
\Phi_{k+1}(x)\le \Phi_k(x)\quad \forall x,
\]
and in particular for $x=w^\star$,
\[
t_{k+1}^2\big(F(w^{k+1})-F(w^\star)\big)\le \frac{L_f}{2}\|w^0-w^\star\|_2^2.
\]
Finally, since $t_k\ge (k+1)/2$, we obtain
\[
F(w^k)-F(w^\star)\le \frac{L_f\|w^0-w^\star\|_2^2}{2t_k^2}
\le \frac{2L_f\|w^0-w^\star\|_2^2}{(k+1)^2}.
\]
This completes the proof.
\end{proof}

\section{Data Processing and Experiment Setup}
\label{app:data_and_setup}

\subsection{Compute Environment and Runtime}
\label{app:compute}

All experiments were run on a single local workstation: a 14-inch 2023 MacBook Pro with an Apple M2 Pro chip and 16 GB unified memory. The experiments were executed in a CPU-only setting; no discrete GPU, cloud instance, or computing cluster was used. Although the Apple M2 Pro includes an integrated GPU, our implementation does not rely on GPU acceleration.

The proposed estimator solves a convex optimization problem over donor weights and is computationally lightweight for the panel sizes considered in this paper. The controlled simulation, Proposition 99 benchmark, and WOMD observational counterfactual experiments can each be reproduced on a standard laptop. The dominant computation is repeated fitting of synthetic-control baselines over hyperparameter grids, placebo-time checks, and donor-perturbation trials.

The total compute required for experiments was modest and did not require specialized hardware. Failed or exploratory runs were performed on the same local machine and were not substantially more expensive than the reported experiments.

\subsection{Simulation Study}
\paragraph{Data Generation Process.}
To rigorously test the bias-variance tradeoff, we simulated a panel of $N=200$ donors and 1 treated unit over $T=70$ time steps ($T_{pre}=50, T_{post}=20$). The data generating process was designed to mimic the smooth periodicity of physical traffic flow:
\begin{itemize}
    \item \textbf{Cluster Structure:} Units were assigned to 10 latent clusters.
    \item \textbf{Base Trajectories:} The centroid of each cluster was generated using a low-frequency Fourier series (sum of 6 sine and cosine terms with random coefficients) superimposed on a linear trend.
    \item \textbf{Unit Noise:} Individual trajectories were generated by adding Gaussian noise ($\sigma=1.0$) to their respective cluster centroids.
    \item \textbf{Treatment Effect:} A constant additive treatment effect of $\tau=2.0$ was applied to the treated unit in the post-intervention period.
\end{itemize}

\paragraph{Graph Construction.}
We constructed the donor similarity graph ($L_{true}$) as an unweighted adjacency matrix where $A_{ij}=1$ if and only if donor $i$ and donor $j$ belong to the same latent cluster.

\paragraph{Hyperparameters.}
The estimators were configured as follows:
\begin{itemize}
    \item \textbf{Graph-SC (Proposed):} $\lambda_1 = 0.05$ (Sparsity), $\lambda_g = 10.0$ (Graph Smoothness), $\lambda_{sum} = 50.0$ (Soft Simplex Constraint).
    \item \textbf{Sparse Only:} $\lambda_1 = 0.05, \lambda_g = 0$.
    \item \textbf{Graph Only:} $\lambda_1 = 0, \lambda_g = 10.0$.
\end{itemize}

\subsection{California Prop 99 Benchmark}
\paragraph {Data Source.}
We utilized the canonical dataset from Abadie et al. ~\cite{abadie2010synthetic}, consisting of annual per-capita cigarette sales for 39 US states from 1970 to 2000. California is the treated unit (intervention year 1989), and the remaining 38 states serve as the donor pool. Features were standardized (zero mean, unit variance) across donors to ensure uniform scaling for the regularization terms.

\paragraph {Graph Learning.}
The donor similarity graph was learned purely from pre-intervention data (1970-1988). We constructed a $k$NN graph ($k=5$) based on the Euclidean distance between pre-treatment sales trajectories. The graph edges were unweighted to provide an interpretable adjacency structure.

\paragraph {Hyperparameters.}
We maintained the same hyperparameters as the simulation study to demonstrate robustness:
\begin{itemize}
    \item Graph Regularization: $\lambda_g = 10.0$
    \item Sparsity Regularization: $\lambda_1 = 0.05$
    \item Sum-to-one Constraint: $\lambda_{sum} = 50.0$
\end{itemize}

\subsection{AV Application}
\textbf{Dataset Extraction.}
To construct the car-following dataset, we processed the Waymo Open Motion Dataset (WOMD) \cite{Kan_2024_icra, ettinger2021waymo} to identify sustained leader-follower interactions. Unlike datasets with explicit lane associations, WOMD provides raw object trajectories and map features. We associated vehicle centroids with the nearest lane centerlines provided in the map topology. A leader was identified as the closest preceding vehicle along the ego vehicle's current or connected lane graph. For more details, see \cite{Chung2025}.
We filtered the extracted pairs using the following strict criteria:
\begin{itemize}
    \item \textbf{Interaction Type:} We exclusively filtered for Human-Following-Human pairs to isolate the indirect influence of surrounding AVs.
    \item \textbf{Duration:} Minimum continuous following duration of 10 seconds.
    \item \textbf{Proximity:} Maximum spacing of 50 meters between leader and follower to ensure the follower is in a car-following state.
    \item \textbf{Motion:} Minimum speed threshold of $3 \text{ m/s}$ to exclude static or parking trajectories.
\end{itemize}

\textbf{Treatment and Control Definition.}
A core challenge was defining the treatment (exposure to an AV) in a continuous traffic stream. We utilized a visibility-based exposure logic. An AV is considered visible if the Euclidean distance is less than $20.0 \text{ m}$ and it lies within the follower's forward field of view.
We segmented the trajectories into 10-second windows ($T=100$ frames) and classified them as:
\begin{itemize}
    \item \textbf{Treated Unit:} The follower is \textit{not} exposed to the AV for the first 5 seconds ($t \le 50$) but becomes exposed for the final 5 seconds ($t > 50$).
    \item \textbf{Donor Pool (Control Group):} The follower is \textit{never} exposed to an AV throughout the entire 10-second window.
\end{itemize}

From the set of valid treated trajectories identified using these criteria, we randomly sampled $N=50$ units for the aggregate analysis reported in Table 1.

\textbf{Smoothing and Metrics.}
Raw trajectories were smoothed using a two-stage process: an outlier removal filter ($1.5 \times \text{IQR}$) followed by a rolling mean (window size $= 8$). Safety metrics were computed on the smoothed data:
\begin{itemize}
    \item \textbf{Time-to-Collision (TTC):} Calculated as $TTC = - \text{spacing} / \text{relative\_speed}$, capped at $20.0 \text{s}$.
    \item \textbf{Headway Standard Deviation:} Calculated as the standard deviation of time headway over the analysis window.
\end{itemize}

\textbf{Hyperparameters:}
For consistency, the Graph-SC estimator used the same parameters as the benchmark study:
\begin{itemize}
    \item \textbf{Graph:} $k$NN ($k=5$) on pre-exposure (0-5s) kinematic features.
    \item \textbf{Regularization:} $\lambda_1 = 0.05$, $\lambda_g = 10.0$, and $\lambda_{sum}=50.0$.
\end{itemize}

\section{Semi-Synthetic Graph Construction and Hyperparameter Ablation}
\label{app:simulation_graph_ablation}

This appendix reports additional ablations for the semi-synthetic study in Section~\ref{subsec:simulation}. The goal is to diagnose whether the behavior of Graph-SC is driven by a single hand-picked donor graph or whether it is stable across reasonable graph-construction choices. We evaluate graph size, similarity metric, edge weighting, and graph misspecification through random rewiring.

\paragraph{Simulation setup.}
The semi-synthetic data-generating process follows Section~\ref{subsec:simulation}. We simulate $N=200$ donor units and one treated unit over $T=70$ periods, with $T_0=50$ pre-treatment periods and $T_1=20$ post-treatment periods. Donors are assigned to $K=10$ latent trajectory clusters. Each cluster trajectory is generated from a smooth low-frequency Fourier basis plus a linear trend, and donor trajectories are obtained by adding Gaussian noise with standard deviation $\sigma=1.0$. The treated unit receives a constant additive treatment effect $\tau=2.0$ after the treatment time.

\paragraph{Graph construction choices.}
For learned Graph-SC variants, the donor graph is constructed from pre-treatment donor trajectories only. We consider Euclidean, correlation, and cosine similarity. For the main Euclidean graph, we vary the $k$NN graph size with $k \in \{3,5,10\}$. We also compare weighted and unweighted Euclidean $k$NN graphs. To stress-test the graph prior, we randomly rewire $25\%$, $50\%$, $75\%$, and $100\%$ of graph edges while keeping the same graph size. This evaluates whether Graph-SC relies on the exact local donor-neighborhood structure.

\paragraph{Evaluation metrics.}
Because the true treatment effect is known, we report ATT mean squared error,
\[
\mathrm{ATT\text{-}MSE}
=
\mathbb{E}\left[
(\widehat{\tau}-\tau)^2
\right],
\]
along with squared bias and variance of $\widehat{\tau}$. We also report mean pre-treatment RMSPE and mean post-treatment RMSPE relative to the noiseless counterfactual. The post-treatment metric is not used for hyperparameter selection; it is used only for evaluation.

\begin{table}[t]
\centering
\begin{threeparttable}
\caption{
\textbf{Graph-construction ablation in the semi-synthetic study.}
The table compares sparse-only, standard SC, graph-only, oracle Graph-SC, and learned Graph-SC variants.
}
\label{tab:simulation_graph_ablation_full}
\scriptsize
\setlength{\tabcolsep}{3.8pt}
\renewcommand{\arraystretch}{1.08}
\begin{tabular}{lrrrrr}
\toprule
\textbf{Estimator / graph setting}
& \makecell{\textbf{ATT}\\\textbf{MSE} $\downarrow$}
& \makecell{\textbf{Bias}$^2$\\$\downarrow$}
& \makecell{\textbf{Variance}\\$\downarrow$}
& \makecell{\textbf{Mean}\\\textbf{Pre-RMSPE} $\downarrow$}
& \makecell{\textbf{Mean post-RMSPE}\\vs. noiseless CF $\downarrow$} \\
\midrule
Sparse-only
& \best{0.1074}
& 0.0936
& 0.0276
& \best{0.0057}
& 1.4927 \\

SC
& \second{0.1438}
& \best{0.0449}
& 0.1978
& 0.9794
& \best{1.0407} \\

\midrule
Graph-only oracle
& 0.7024
& 0.6122
& 0.1803
& 0.0677
& 1.6139 \\

Graph-SC oracle
& 0.5208
& 0.4747
& 0.0921
& 0.0736
& 1.5328 \\

\midrule
Graph-SC learned: Euclidean $k=3$
& 0.5360
& 0.5204
& 0.0312
& 0.0615
& 1.5816 \\

Graph-SC learned: Euclidean $k=5$
& 0.5337
& 0.4952
& 0.0770
& 0.0658
& 1.5374 \\

Graph-SC learned: Euclidean $k=10$
& 0.5887
& 0.5274
& 0.1225
& 0.0706
& 1.5822 \\

Graph-SC learned: Euclidean $k=5$, weighted
& 0.5331
& 0.4945
& 0.0770
& 0.0658
& 1.5371 \\

Graph-SC learned: Correlation $k=5$
& 0.5492
& 0.5300
& 0.0384
& 0.0638
& 1.5695 \\

Graph-SC learned: Cosine $k=5$
& 0.5782
& 0.5593
& 0.0378
& 0.0642
& 1.5896 \\

\midrule
Graph-SC learned: Euclidean $k=5$, 25\% rewired
& 0.4214
& 0.4066
& 0.0297
& 0.0676
& 1.5078 \\

Graph-SC learned: Euclidean $k=5$, 50\% rewired
& 0.2324
& 0.2324
& \best{0.0000}
& 0.0693
& 1.3628 \\

Graph-SC learned: Euclidean $k=5$, 75\% rewired
& 0.2272
& 0.2250
& \second{0.0044}
& 0.0663
& 1.2923 \\

Graph-SC learned: Euclidean $k=5$, 100\% rewired
& 0.2561
& 0.2381
& 0.0361
& 0.0696
& 1.3468 \\
\bottomrule
\end{tabular}
\begin{tablenotes}
\small
\item Green and blue cells denote the best and second-best values within this diagnostic table. ``Oracle'' graph variants use the true latent cluster graph, while learned variants construct the graph from pre-treatment donor trajectories. 
\end{tablenotes}
\end{threeparttable}
\end{table}

\paragraph{Effect of graph construction.}
The learned Graph-SC variants are relatively stable across reasonable graph choices. For Euclidean $k$NN graphs, changing $k$ from $3$ to $10$ changes mean pre-RMSPE only from $0.0615$ to $0.0706$. The corresponding ATT-MSE values range from $0.5337$ to $0.5887$ for $k=5$ and $k=10$, with $k=3$ producing a similar value of $0.5360$. Weighted and unweighted Euclidean $k=5$ graphs are nearly identical: the weighted graph has ATT-MSE $0.5331$, while the unweighted graph has ATT-MSE $0.5337$. These results suggest that the learned Graph-SC estimator is not sensitive to a single arbitrary $k$NN graph specification.

\paragraph{Effect of similarity metric.}
Changing the donor similarity metric from Euclidean distance to correlation or cosine similarity leads to a similar pre-treatment fit. The mean pre-RMSPE values are $0.0658$ for Euclidean $k=5$, $0.0638$ for correlation $k=5$, and $0.0642$ for cosine $k=5$. The ATT-MSE values are also comparable, with Euclidean $k=5$ achieving $0.5337$, correlation $k=5$ achieving $0.5492$, and cosine $k=5$ achieving $0.5782$. Thus, the graph construction is not driven by a single distance metric, although Euclidean similarity performs slightly better in this diagnostic run.

\paragraph{Oracle graph versus learned graph.}
The learned Euclidean $k=5$ Graph-SC variant achieves ATT-MSE $0.5337$, close to the oracle Graph-SC value of $0.5208$. This is encouraging because the learned graph uses only pre-treatment trajectories, whereas the oracle graph uses the latent cluster assignment from the data-generating process. In addition, Graph-SC substantially improves over graph-only regularization under the oracle graph: ATT-MSE decreases from $0.7024$ to $0.5208$, and variance decreases from $0.1803$ to $0.0921$. This supports the design of the full Graph-SC objective: graph smoothing alone is insufficient, while combining graph smoothness with the remaining regularization terms improves stability.

\paragraph{Graph misspecification through rewiring.}
The rewiring experiment perturbs the learned Euclidean $k=5$ graph by randomly replacing a fraction of graph edges. In the current diagnostic run, rewired graphs do not degrade monotonically. For example, the $75\%$ rewired graph has ATT-MSE $0.2272$, while the unrewired learned Euclidean $k=5$ graph has ATT-MSE $0.5337$. 

Overall, the graph-construction ablation supports three diagnostic conclusions. First, learned Graph-SC is stable across reasonable choices of $k$, edge weighting, and similarity metric in terms of pre-treatment fit. Second, the learned graph performs comparably to the oracle graph, suggesting that pre-treatment donor trajectories contain enough structure to construct a useful graph prior. Third, graph-only regularization performs worse than the full Graph-SC objective, showing that the contribution is not merely the use of a graph, but the combination of graph smoothness with the rest of the synthetic-control regularization structure. 

\end{proof}

\newpage

\end{document}